%% file: main.tex
\documentclass[letterpaper, 10 pt, journal, twoside]{IEEEtran}

\usepackage{times} 
\usepackage{amsmath} 
\usepackage{amssymb}  

\usepackage[utf8]{inputenc}
\usepackage{hyperref}
\usepackage{booktabs}
\usepackage{multirow}
\usepackage{subcaption} 
\usepackage{array}
\usepackage{xspace}
\usepackage{algorithm}
\usepackage{algorithmicx}
\usepackage{tabularx}
\usepackage{graphicx}
\usepackage{placeins}
\usepackage{centernot}
\usepackage{siunitx}

\usepackage[noend]{algpseudocode}

\hypersetup{urlcolor=blue, colorlinks=true}

\newcommand{\sm}[1]{{\leavevmode\color{black}#1}}

\newcommand{\state}{\ensuremath{\mathbf{s}}\xspace}
\newcommand{\States}{\ensuremath{\mathcal{S}}\xspace}
\newcommand{\plan}{\ensuremath{\pi}\xspace}
\newcommand{\Pvec}{\ensuremath{\Pi}\xspace}
\newcommand{\goalreg}{\ensuremath{\mathcal{G}}\xspace}
\newcommand{\open}{\ensuremath{OPEN}\xspace}
\newcommand{\edge}{\ensuremath{(\mathbf{s},\mathbf{a})}\xspace}

\newcommand{\graph}{\ensuremath{G}\xspace}
\newcommand{\vertex}{\ensuremath{v}\xspace}
\newcommand{\ed}{\ensuremath{e}\xspace}
\newcommand{\Vertices}{\ensuremath{\mathcal{V}}\xspace}
\newcommand{\Edges}{\ensuremath{\mathcal{E}}\xspace}
\newcommand{\eclosed}{\ensuremath{E^{closed}}\xspace}
\newcommand{\eopen}{\ensuremath{E^{open}}\xspace}
\newcommand{\eeval}{\ensuremath{E^{eval}}\xspace}

\newcommand{\action}{\ensuremath{\mathbf{a}}\xspace}
\newcommand{\Aset}{\ensuremath{\mathcal{A}}\xspace}
\newcommand{\cost}{\ensuremath{c}\xspace}
\newcommand{\costreal}{\ensuremath{\cost^t}\xspace}
\newcommand{\costlazy}{\ensuremath{\cost}\xspace}
\newcommand{\costopt}{\ensuremath{\cost^*}\xspace}
\newcommand{\thread}{\ensuremath{\text{T}}\xspace}
\newcommand{\numthreads}{\ensuremath{N_t}\xspace}
\newcommand{\gval}{\ensuremath{g}\xspace}
\newcommand{\wh}{\ensuremath{w}\xspace}
\newcommand{\goalg}{\ensuremath{\cost^{bound}}\xspace}
\newcommand{\comment}[1]{}

\newcommand{\pick}{Pick\xspace}
\newcommand{\place}{Place\xspace}
\newcommand{\cci}{\ensuremath{d_{cc}}\xspace}
\newcommand{\tsim}{\ensuremath{t_s}\xspace}
\newcommand{\tplan}{\ensuremath{t_p}\xspace}

\newtheorem{theorem}{Theorem}
\newtheorem{lemma}[theorem]{Lemma}
\newenvironment{proof}{{\bfseries Proof}}{}

\algdef{SE}[VARIABLES]{Variables}{EndVariables}
   {\algorithmicvariables}
   {\algorithmicend\ \algorithmicvariables}
\algnewcommand{\algorithmicvariables}{\textbf{global variables}}
\usepackage[font=small]{caption}

\begin{document}
\title{MPLP: Massively Parallelized Lazy Planning}

\author{Shohin Mukherjee, Sandip Aine, Maxim Likhachev
\thanks{This work was supported by the ARL-sponsored A2I2 program, contract W911NF-18-2-0218, and ONR grant N00014-18-1-2775.}
\thanks{The authors are with the Robotics Institute, Carnegie Mellon University, Pittsburgh, PA 15213, USA. E-mail:
        {\tt\small \{shohinm, asandip, mlikhach\}@andrew.cmu.edu}}
}


\maketitle

\begin{abstract}
\input{01abstract}
\end{abstract}

\FloatBarrier
\section{Introduction}
\input{02introduction}

\FloatBarrier
\section{Related Work}
\input{03background}

\FloatBarrier
\section{Method}
\label{sec:methods}
\input{04methods}

\FloatBarrier
\section{Properties}
\label{sec:Properties}
\input{05properties}

\FloatBarrier
\section{Evaluation}
\input{06evaluation}

\FloatBarrier
\section{Discussion and Future Work}
\input{07conclusions}

\FloatBarrier
\bibliographystyle{IEEEtran}
\bibliography{main}

\end{document}

%% file: 01abstract.tex
Lazy search algorithms have been developed to efficiently solve planning problems in domains where the computational effort is dominated by the cost of edge evaluation. The existing algorithms operate by intelligently balancing computational effort between searching the graph and evaluating edges. However, they are designed to run as a single process and do not leverage the multithreading capability of modern processors. In this work, we propose a massively parallelized, bounded suboptimal, lazy search algorithm (MPLP) that harnesses modern multi-core processors. In MPLP, searching of the graph and edge evaluations are performed completely asynchronously in parallel, leading to a drastic improvement in planning time. We validate the proposed algorithm in two different planning domains: 1) motion planning for 3D humanoid navigation and 2) task and motion planning for a robotic assembly task. We show that MPLP outperforms the state-of-the-art lazy search as well as parallel search algorithms. The open-source code for MPLP is available here: \url{https://github.com/shohinm/parallel_search}

%% file: 02introduction.tex
Graph search algorithms such as A* and its variants~\cite{hart1968formal, pohl1970heuristic, aine2016multi} are widely used in robotics for task and motion planning problems~\cite{kusnur2021planning, mukherjee2021reactive} which can be formulated as a shortest path problem on an embedded graph in the state-space of the domain. The computational cost of solving the shortest path problem can be split between the cost of traversing through and searching the graph (discovering states, maintaining and managing ordered data structures, rewiring vertices, etc.) and evaluating the cost of the edges. In robotics applications such as in motion planning, the edge evaluation tends to be the bottleneck of solving the shortest path problem. For example, in planning for robot-manipulation, edge evaluation typically corresponds to collision-checks of a robot model against the world model at discrete interpolated states on the edge. Depending on how these models are represented (meshes, spheres, etc.) and how finely the states to be collision checked are interpolated, evaluating an edge can get quite expensive.  

To address this issue, lazy search algorithms~\cite{narayanan2017heuristic, dellin2016unifying, haghtalab2018provable, mandalika2018lazy} have been developed that defer the evaluation of discovered edges and instead use \textit{estimates} of edge costs to search the graph whenever the true edge costs are \textit{unknown}. \sm{Here, \textit{knowing} the true edge cost implies running the computation to evaluate the edge which is often expensive}. Instead, the estimate is an easier-to-compute approximation of the true edge cost. Various lazy search algorithms mainly differ in the way they toggle between searching the graph and evaluating the edges. The benefit of these methods is that they are more time-efficient in domains where the cost of edge evaluation outweighs the cost of the search. In all of the current lazy search algorithms, performance depends on two critical design choices: 1) the strategy employed to toggle between searching the graph and evaluating the edges, and 2) the order in which edges are evaluated~\cite{mandalika2019generalized}. This is because these algorithms are designed to run as a single process.

Our key insight is that instead of toggling between searching the graph and evaluating the edges, these operations can happen asynchronously in parallel. This allows us to harness the massive parallelization capabilities of modern processors.  We propose a new algorithm: Massively Parallel Lazy Planning (MPLP), that leverages this insight. MPLP eliminates the need for an explicit strategy to balance computational effort between the search and edge evaluations by parallelizing these two operations. On the theoretical front, we show that MPLP provides rigorous guarantees of optimality or bounded suboptimality if heuristics are inflated as in Weighted A*~\cite{pohl1970heuristic}. MPLP can be used for any planning problem with expensive to evaluate edges and run efficiently on any processor that supports multiprocessing.  We show this by evaluating and comparing MPLP against lazy search and parallel search baselines on two planning problems: 1) 3D indoor navigation of a humanoid and 2) a task and motion planning problem of stacking a set of blocks by a robot. All experiments were carried out on AWS instances running up to 90 cores in parallel. The experimental results show that by combining ideas from lazy search and parallel search, MPLP achieves higher time efficiency than existing lazy search algorithms such as LwA*~\cite{cohen2014planning} and LSP~\cite{dellin2016unifying}, as well as parallel search algorithms such as PA*SE~\cite{phillips2014pa}.

%% file: 03background.tex
MPLP is inspired by two categories of planning algorithms that achieve increased efficiency in different ways i.e. lazy search and parallel search. 

\textbf{Lazy search:} Lazy search algorithms achieve greater time efficiency than regular graph search algorithms in domains where the planning time is dominated by edge evaluations. They do so by deferring the evaluation of edges generated during the search and proceeding with the search using cheap-to-compute estimates of the edge costs for the unevaluated edges. Different lazy search algorithms differ in how they toggle between searching the graph and evaluating the edges as well as the order in which the edges are evaluated. In A*, when a state is expanded, all outgoing edges are immediately evaluated. In Lazy Weighted A*~(LwA*)~\cite{cohen2014planning}, when a state is expanded, the outgoing edges are not evaluated, instead, the successors are added to the open list with cheap-to-compute underestimates of the true edge costs. When these states are expanded, only the incoming edges that connect their best predecessors are evaluated. In Lazy Shortest Path~(LSP)~\cite{dellin2016unifying}, the search proceeds without evaluating any edge until the goal is expanded. It then evaluates the edges that are on the shortest path, updates the costs of these edges, and replans, until a path is found with no unevaluated edges. Lazy Receding Horizon A*~(LRA*)~\cite{mandalika2018lazy} allows the search to proceed to an arbitrary lookahead before evaluating edges. In~\cite{mandalika2019generalized}, a general framework for lazy search algorithms called Generalized Lazy Search (GLS) was formulated. It was shown that by employing different strategies to toggle between searching the graph and evaluating edges, as well as choosing the order in which the edges are evaluated, various lazy search algorithms can be recovered. GLS can also leverage priors on edge validity to come up with more efficient policies that minimize planning time. In~\cite{lim2021lazy}, ideas from GLS and incremental methods like LPA* were integrated into a lazy lifelong planning algorithm. There has also been work on anytime algorithms that leverage edge existence priors to come up with a strategy to evaluate edges, such that the suboptimality bound on the solution quality is minimized in expectation of the algorithm interruption time while reducing planning time~\cite{narayanan2017heuristic}. In contrast to all these methods, MPLP harnesses the parallelization capabilities of modern processors by running the search and the edge evaluations completely asynchronously. Therefore there is no need to devise an explicit strategy to toggle between these two components of the search and it makes the entire planning more efficient.

\textbf{Parallel search:} Parallel search algorithms on the other hand seek to make planning faster by leveraging parallelization. There are a number of approaches that parallelize sampling-based planning algorithms. Probabilistic roadmap (PRM) based methods, in particular, can be trivially parallelized, so much so that they have been described as ``embarrassingly parallel"~\cite{amato1999probabilistic}. In these approaches, several parallel processes cooperatively build the roadmap in parallel~\cite{jacobs2012scalable}. Parallelized versions of RRT have also been developed in which multiple cores expand the search tree by sampling and adding multiple new states in parallel~\cite{devaurs2011parallelizing, jacobs2013scalable, park2016parallel}. In this work, however, we focus on search-based planning. 

A trivial approach to achieve parallelization in weighted A* is to generate successors in parallel when expanding a state. The downside is that this leads to minimal improvement in performance in domains with a low branching factor. Another approach that Parallel A*~\cite{irani1986parallel} takes, is to expand states in parallel while allowing re-expansions to account for the fact that states may get expanded before they have the minimal cost from the start state. This leads to a high number of state expansions. There are a number of other approaches that employ different parallelization strategies~\cite{zhou2015massively, burns2010best}, but all of them have could potentially expand an exponential number of states, especially if they employ a weighted heuristic. In contrast, PA*SE~\cite{phillips2014pa} expands states in parallel that are independent of each other in such a way that does not affect the bounds on the solution quality. This approach also has limited parallelization though, especially when the search aims for a tight suboptimality bound since then there are only a few states that are independent and can be safely expanded in parallel. There has also been work on parallelizing  A* search on a single GPU~\cite{zhou2015massively} or multiple GPUs~\cite{he2021efficient} by utilizing multiple parallel priority queues to parallely expand states. \sm{However, GPUs have a single-instruction-multiple-data (SIMD) execution model, which means that they can only run the same set of instructions on multiple data concurrently. This severely limits the design of planning algorithms in several ways. Firstly, the code for expanding a state must be identical, irrespective of what state is being expanded. Secondly, the set of states must be expanded in a batch. This is problematic in domains that have complex actions that correspond to forward simulating dissimilar controllers. In contrast to these approaches, MPLP achieves massive parallelization (shared-memory parallelization) of edge evaluations on the CPU which has a multiple-instruction-multiple-data (MIMD) execution model, which allows it to parallelize potentially dissimilar edges, and therefore generalize across all types of planning domains.}


%% file: 04methods.tex
Typical search-based planning algorithms like A* proceed by expanding states till an (optimal) path to the goal is obtained. When a state is expanded, its successors are generated by applying the actions in the action space of the domain, which are represented by edges in a graph. When a successor is generated, the edge between the expanded state and the successor is evaluated for computing the edge cost, which is then used in the search. However, in domains where edge evaluation is expensive, the search slows down dramatically. In lazy search methods, edge evaluation is deferred and an \textit{optimistic} (a fast-to-compute underestimate of the actual cost) estimate of the cost is used instead. The key idea behind MPLP is to run the search using optimistic costs while evaluating relevant edges using a pool of threads entirely asynchronously in parallel. To make this process effective, the research questions are what edges to evaluate, in what order, and how to incorporate these evaluations into the search.

\subsection{Problem Formulation}

Let a finite graph $\graph = (\Vertices, \Edges)$ be defined as a set of vertices \Vertices and edges \Edges. Each vertex $\vertex \in \Vertices$ represents a state \state in the state space of the domain \States. An edge $\ed \in \Edges$ connecting two vertices $\vertex_1$ and $\vertex_2$ in the graph represents an action $\action \in \Aset$ that takes the agent from corresponding states $\state_1$ to $\state_2$. In this work, we assume that all actions are deterministic. Hence an edge \ed can be represented as a pair $\edge$, where \state is the state at which action \action is executed. Each edge has an associated true cost $\costreal:\Edges \rightarrow \sm{[0,\infty]}$ which can be computed using a typically expensive edge evaluation routine.  A feasible edge is an edge with a finite true cost, i.e. $\costreal(\ed) < \infty$. In addition, there is an optimistic cost associated with each edge $\costlazy:\Edges \rightarrow [0,\infty]$ that is easy to compute and underestimates the true cost i.e. $\costlazy(\ed) \leq \costreal(\ed)$.  Let $\Edges^{eval} \subset \Edges$ be a the subset of edges which have been evaluated and hence for which the true costs are available.

A path \plan is defined by an ordered sequence of edges $\edge_{i=1}^N$, the true cost of which is denoted as $\costreal(\plan) = \sum_{i=1}^N \costreal(\ed_i)$. A feasible path is a path with no infeasible edges, therefore $\costreal(\plan) < \infty$. In addition, we define an optimistic cost for a path that is not \textit{fully evaluated} (i.e. not all the edges in the path have been evaluated) using the true cost for the evaluated edges and the optimistic cost otherwise i.e.

$$\cost(\plan) =  \sum_{i=1}^N   
\begin{cases}
      \costreal(\ed_i), & \text{if}\ \ed_i \in \Edges^{eval} \\
      \costlazy(\ed_i), & \text{otherwise}
\end{cases} 
$$ 

MPLP seeks to find a path \plan from a given start state $\state_0$ to a goal region \goalreg comprising of only evaluated edges such that the true cost of the path satisfies the relationship $\costreal(\plan)~\leq~\epsilon~\cdot~\cost^*$, where $\cost^*$ is the optimal cost from $\state_0$ to \goalreg and $\epsilon \geq 1$ is suboptimality bound. There is a computational budget of \numthreads threads available which can run in parallel.

\subsection{Algorithm}

\input{mplp.tex}

\begin{figure}[ht]
    \centering
    \includegraphics[width=\columnwidth]{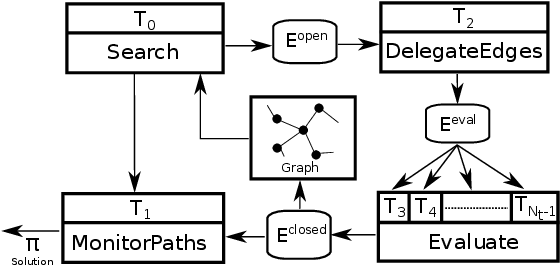}
    \caption{The figure depicts a high-level overview of MPLP. The search runs on thread $\thread_0$. The discovered edges are added to a priority queue \eopen. On thread $\thread_2$, \textsc{DelegateEdges} delegates the evaluation of edges in \eopen to a thread from a pool of threads ($\thread_{i=3:\numthreads}$) dedicated to edge evaluation. The edges under evaluation are moved from \eopen to a list \eeval. When an edge has been evaluated, it is added to a list \eclosed and the graph is updated to reflect the true edge cost. \textsc{MonitorPaths} monitors the generated paths in thread  $\thread_1$ and waits for a path that is fully evaluated (all edges in \eclosed) and satisfies a suboptimality check. Upon finding such a path, it returns it as the solution and the algorithm is terminated.}
    \label{fig:algo_overview}
\end{figure}

\subsubsection*{\textbf{Overview}}
MPLP runs three key aspects of the search asynchronously in parallel: 1) the optimistic search (a search that uses the true cost for the evaluated edges and the optimistic cost otherwise), 2) edge evaluations, and 3) an evaluation status and suboptimality check on the paths generated by the optimistic search. The suboptimality check is explained in Sec.~\ref{sec:Properties}. MPLP allocates the given budget of \numthreads threads to these aspects as illustrated in Fig.~\ref{fig:algo_overview}.

The optimistic search runs on thread $\thread_0$ as an iterative sequence of weighted A*~\cite{likhachev2003ara} searches which proceed without evaluating any edge. Whenever a new edge is discovered, it is added to a priority queue and scheduled for evaluation in order of its priority (higher priority first), and the search proceeds with an optimistic underestimate of the true edge cost. For edges that have already been evaluated by the edge evaluation threads, the search uses the true cost. \sm{Initially, all newly discovered edges have the same priority of 1 for evaluation, in which case the edges in the priority queue follow FIFO ordering. When the search finds a path to the goal, the evaluation priorities of the unevaluated edges in the path are dynamically \sm{increased} to 2.} This is to ensure that the edges that belong to a path to the goal are evaluated before the other edges that are discovered during the search, which depending on the size of the graph and the greediness of the search can be numerous (see ablation in Sec.~\ref{sec:priority_ablation}). Though we use this naive priority update in the current version of the algorithm, more intelligent strategies can potentially be employed.

Another thread $\thread_2$ acts as a delegator of edges awaiting evaluation and delegates the edges in the queue to a pool of threads ($\thread_{i=3:\numthreads}$) dedicated to edge evaluation. Whenever any of these threads finishes evaluating an edge, it updates the graph with the true edge cost. Finally, another thread $\thread_1$ monitors the state of every path that has been found by the optimistic search and when it finds a path that has been fully evaluated and satisfies a suboptimality bound (Theorem~\ref{th:optimality}, Sec.~\ref{sec:Properties}), it returns it as the solution which terminates the algorithm. Running this asynchronously allows $\thread_0$ to proceed immediately to the next search iteration without waiting for the generated path to be evaluated and undergo the suboptimality check. 

Because of the asynchronous operation of MPLP, the edge evaluation threads $\thread_{i=3:\numthreads}$ can update the graph in the middle of an ongoing search on $\thread_0$. \comment{The optimistic search runs a sequence of weighted A* searches (Line~\ref{alg:search/lower_g_val}, Alg.~\ref{alg:search})~\cite{likhachev2003ara}.} When any single weighted A* search on $\thread_0$ terminates, the resulting path is a solution on an implicit snapshot of the graph in which the cost of each edge is its cost at the time the source state of the edge was expanded during the search. This may be the true edge cost or the optimistic underestimate depending on whether the edge was evaluated by an edge evaluation thread.

\subsubsection*{\textbf{Details}}

Besides an open list (\open) for the states, MPLP uses the following data structures for the edges: A priority queue of edges that need to be evaluated (\eopen), a list of edges that have been evaluated (\eclosed) and a list of edges that are under evaluation (\eeval). Unlike in \open where states with smaller keys are placed in the front of the queue, in \eopen edges with higher priorities are placed in front. The optimistic search (Alg.~\ref{alg:search}) runs an iterative sequence of weighted A* searches from scratch (\textsc{ComputePath}) in a loop (Line~\ref{alg:search/compute_path}) as a single process on thread $\thread_0$. For every state that is expanded for the first time, its successors are generated in Line~\ref{alg:search/gen_successors} but the corresponding edges are not evaluated, instead, they are added to \eopen. The search then proceeds with the optimistic edge cost \costlazy for the unevaluated edges, but for edges that have already been evaluated, it uses the true cost \costreal. We use a constant priority of 1 for all edges when they are first inserted into \eopen. When a state in the goal region \goalreg is expanded in Line~\ref{alg:search/goal_reached}, the path \plan obtained by backtracking from the goal state to $\state_0$ is added to a list of generated paths \Pvec in \textsc{ConstructPath}. The priorities of unevaluated edges in \plan are \sm{increased} by a multiplicative factor of 2 in \eopen to prioritize the evaluation of the edges in the paths (Line~\ref{alg:search/inflate_priority}). In addition, the maximum of the costs of the paths generated by \textsc{ComputePath} is stored in a variable \goalg. As explained in Sec.~\ref{sec:Properties}, \goalg is upper bounded by $\wh\cdot\costopt$, where $\wh$ is the heuristic inflation factor and $\costopt$ is the cost of an optimal path in \graph. 


In a separate thread $\thread_2$, \textsc{DelegateEdges} (Alg.~\ref{alg:edge_evaluation}) delegates the evaluation of edges in \eopen in order of their priorities to a  pool of threads ($\thread_{i=3:\numthreads}$) dedicated to edge evaluations (Line~\ref{alg:edge_evaluation/delegate}). When a thread is available, the edge is evaluated in Line~\ref{alg:edge_evaluation/evaluate} to obtain the true cost \costreal. When an edge is being evaluated, it is moved from \eopen to \eeval. Once it has been evaluated, it is moved to \eclosed. If the true cost is different from the estimated cost, the graph is updated (Line~\ref{alg:edge_evaluation/update}). A generated edge, therefore, belongs to one of the three containers i.e. \eopen, \eeval or \eclosed, at any point in time. \sm{Therefore, if a state is revisited, the state along with its incoming edge need not be regenerated (Line~\ref{alg:search/retrieve_successors}, Alg~\ref{alg:search}).}


Another asynchronous process \textsc{MonitorPaths} (Alg.~\ref{alg:monitor_paths}) on thread $\thread_2$ monitors the state of every edge in the paths in \Pvec. If a path is found that has been fully evaluated and that has cost no greater than \goalg, it is returned as the solution. This check is necessary to guarantee bounded suboptimality as proved in section \ref{sec:Properties}. The optimistic search terminates when either \textsc{MonitorPaths} finds a solution and sets the variable $solution\_found$ (Line~\ref{alg:monitor_paths/solution_found}, Alg.~\ref{alg:monitor_paths}) or when \textsc{ComputePath} terminates without a path by exhausting the open list ($path\_exists$ is false). An implementation detail to note is that \eopen is modified by multiple threads asynchronously. Therefore, to ensure thread safety by protecting against data race, \eopen must be accessed under a synchronization lock.

\subsubsection*{\textbf{Discussion}}
MPLP has some key differences from other lazy search algorithms.
\begin{itemize}
    \item The search and edge evaluations run completely asynchronously. Unlike in the GLS framework, there is no explicit strategy employed to toggle between the search and edge evaluations.
    \item All other lazy search algorithms like LwA*, LSP and LRA* only evaluate edges that are either on the shortest path to the goal, or likely to be so. This is to ensure that computational effort is not wasted on evaluating edges not likely to be on the shortest path. MPLP on the other hand evaluates every edge that the search encounters while prioritizing edges that are on the shortest paths in the partially evaluated graphs. This allows it to exploit massive parallelization.
\end{itemize}

%% file: mplp.tex
\begin{algorithm}[]
\caption{\label{alg:search} MPLP: Search ($\thread_0$)}
\begin{footnotesize}
\begin{algorithmic}[1]
\State $\Aset\gets\text{ action space }$, $\numthreads \gets$ number of threads
\Comment{Shared variables}
\State $\graph \gets \emptyset$, $\state_0\gets\text{ start state }$, $\goalreg\gets\text{ goal region}$
\State $\eopen \gets \emptyset$, $\eclosed \gets \emptyset$, $\eeval \gets \emptyset$, \sm{$\Pvec \gets \emptyset$}
\State $solution\_found \gets \text{False}$, $terminate \gets \text{False}$, $\goalg\gets-\infty$
\Procedure{Mplp}{}
    \State $path\_exists \gets \text{True}$
    \State Spawn \textsc{MonitorPaths} on $\thread_1$ and \textsc{DelegateEdges} on $\thread_2$
    \While{$\textbf{not } solution\_found \textbf{ and }  path\_exists$}
        \State $path\_exists = \textsc{ComputePath}(\state_0, \goalreg)$
        \label{alg:search/compute_path}
    \EndWhile
    \State $terminate = \text{True}$
    \State $\Return~solution\_found$
\EndProcedure
\Procedure{ComputePath}{$\state_0, \goalreg$} 
    \State $\forall\state\in\graph$, $\state.\gval\gets\infty$, $\state.is\_closed=\text{False}$
    \Comment{Reset discovered states}
    \State $\state_0.\gval\gets0$, $\open \gets \emptyset$, \sm{\open.\textsc{Push}($\state_0$, \textsc{Key}($\state_0$))}
    \While{$\open \neq \emptyset$}
        \State $\state \gets \open.\textsc{Pop}()$
        \If{$\state \in \goalreg$}
            \label{alg:search/goal_reached}
            \Comment{Goal reached}
            \State $\goalg = \max(\goalg, \state.\gval)$
            \label{alg:search/goalg_update}
            \State $\textsc{ConstructPath(\state)}$
            \State $\Return~\text{True}$
        \Else 
            \State \textsc{Expand}$(\state)$
            \State $\state.is\_closed = \text{True}$
        \EndIf
    \EndWhile
    \State $\Return~\text{False}$ 
\EndProcedure
\Procedure{ConstructPath}{$\state$}
    \State $\plan \gets \emptyset$
    \While{$\state \neq \state_0$}
        \State $(\state^P, \action^P) \gets \state.\textsc{GetParent}()$
        \State LOCK
        \If{$(\state^P, \action^P) \in \eopen$}
            \State \sm{$\eopen.\textsc{Update}((\state^P, \action^P), 2)$}
            \label{alg:search/inflate_priority}
            \Comment{Increase priority to 2}
        \EndIf
        \State UNLOCK
        \State $\plan.\textsc{Append}((\state^P, \action^P))$
        \State $\state\gets\state^P$
    \EndWhile
    \If{$\plan \centernot\in \Pvec$}
    \label{alg:search/unique_check}
        \State $\Pvec.\textsc{Append}(\plan)$    
    \EndIf
\EndProcedure
\Procedure{Expand}{$\state$}
    \For{$\action \in \Aset$}
        \If{$\edge \in \eopen \cup \eeval \cup \eclosed$}
            \State $(\state', \cost) \gets \graph.\textsc{GetSuccessor}(\state, \action)$
            \label{alg:search/retrieve_successors}
        \Else
            \State $(\state', \cost) \gets \textsc{GenerateSuccessor}(\state, \action)$
            \label{alg:search/gen_successors}
            \State $\graph.\textsc{AddEdge}(\edge, \cost)$
            \State LOCK
            \State \sm{$\eopen.\textsc{Push}(\edge, 1)$}
            \Comment{Initial priority of 1}
            \label{alg:search/eopen_insert}
            \State UNLOCK
        \EndIf
        \If{$\textbf{not}~\state'.is\_closed \textbf{ and } \state.\gval + \cost < \state'.\gval$}
            \State $\state'.\gval = \state.\gval + \cost$
            \label{alg:search/lower_g_val}
            \State $\state'.\textsc{SetParent}(\edge)$
            \State \open.\textsc{Push}$(\state', \textsc{Key}(\state'))$
        \EndIf
    \EndFor
\EndProcedure
\Procedure{Key}{$\state$}
    \State $\Return~\state.\gval+\wh\cdot\textsc{GetHeuristic}(\state)$
\EndProcedure
\end{algorithmic}
\end{footnotesize}
\end{algorithm}

\begin{algorithm}[]
\caption{\label{alg:edge_evaluation} MPLP: Edge Evaluation ($\thread_2, \thread_{i=3:\numthreads}$)}
\begin{footnotesize}
\begin{algorithmic}[1]
\Procedure{DelegateEdges}{} 
    \While{$\textbf{not } terminate$}
        \For{$i=3:\numthreads$}
            \If{$\thread_i$ is available \textbf{and} $\eopen \neq \emptyset$}
                \State LOCK
                \State $\edge \gets \eopen.\textsc{Pop}()$
                \label{alg:edge_evaluation/eopen_pop}
                \State UNLOCK
                \State $\eeval.\textsc{Insert}(\edge)$
                \State Spawn $\textsc{Evaluate}(\edge)$ on $\thread_i$
                \label{alg:edge_evaluation/delegate}
            \EndIf
        \EndFor
    \EndWhile
\EndProcedure
\Procedure{Evaluate}{$\edge$}
    \State $\costreal \gets \textsc{EvaluateEdge}(\edge)$
    \label{alg:edge_evaluation/evaluate}
    \If{$\graph.\textsc{Cost}(\edge) \neq \costreal$}
        \State $\graph.\textsc{UpdateEdgeCost}(\edge, \costreal)$
        \label{alg:edge_evaluation/update}
    \EndIf
    \State $\eeval.\textsc{Remove}(\edge)$
    \State $\eclosed.\textsc{Insert}(\edge)$
\EndProcedure
\end{algorithmic}
\end{footnotesize}
\end{algorithm}

\begin{algorithm}[]
\caption{\label{alg:monitor_paths} MPLP: Monitor Paths ($\thread_1$)}
\begin{footnotesize}
\begin{algorithmic}[1]
\Procedure{MonitorPaths}{}
    \While{$\textbf{not } terminate$}
        \For{$\plan \in \Pvec$}
            \State $path\_evaluated \gets \text{True}$, $\cost^{\plan} \gets 0$
            \For{$\edge \in \plan$}
                \If{$\edge \in \eclosed$}
                    \State $\cost^{\plan} =\cost^{\plan}+\graph.\textsc{Cost}(\edge)$
                \Else
                     \State $path\_evaluated = \text{False}$
                     \State \textbf{break}
                \EndIf
            \EndFor
            \If{$path\_evaluated$}
                \If{$\cost^{\plan} \leq \goalg$}
                    \label{alg:monitor_paths/feasibility_check}
                    \State $solution\_found = \text{True}$
                    \label{alg:monitor_paths/solution_found}
                    \State $\Return~\plan$ 
                \Else
                    \State $\Pvec.\textsc{Remove}(\plan)$
                \EndIf
            \EndIf
        \EndFor
    \EndWhile
\EndProcedure
\end{algorithmic}
\end{footnotesize}
\end{algorithm}

%% file: 05properties.tex
MPLP is guaranteed to be complete and bounded suboptimal and we prove these properties.

\begin{lemma}
\label{lem:monitor_paths}
If there exists a path \plan in \Pvec that is fully evaluated and satisfies the suboptimality bound (Line~\ref{alg:monitor_paths/feasibility_check}, Alg.~\ref{alg:monitor_paths}), \textsc{MonitorPaths} will return a solution in finite time.
\end{lemma}
\begin{proof}
\sm{Since Alg.~\ref{alg:search} runs weighted A*, the paths computed by it are cycle-free. For a finite graph \graph, there are a finite number of cycle-free paths (possibly with a mix of evaluated and unevaluated edges) from $\state_0$ to \goalreg.} Moreover, the uniqueness check in Line~\ref{alg:search/unique_check} of Alg.~\ref{alg:search} ensures that there are no duplicate paths in \Pvec. Therefore \Pvec is of finite size and since \textsc{MonitorPaths} iterates over \Pvec repeatedly, it is bound to discover a path \plan in \Pvec that is fully evaluated and satisfies the suboptimality bound in finite time, if such a path exists in \Pvec.

\end{proof}

\begin{theorem}[\textbf{Completeness}]
\label{th:completness}
If there exists at least one feasible path \plan in \graph from $\state_0$ to \goalreg, MPLP will return a solution in finite time.
\end{theorem}
\begin{proof}
MPLP runs a sequence of weighted A* searches on a finite graph \graph, which is a complete algorithm. The edges in \graph are being simultaneously evaluated by \textsc{EvaluateEdges}. In the worst case, the optimistic search will have discovered all edges in \graph and added them to \eopen (Line~\ref{alg:search/eopen_insert}, Alg.~\ref{alg:search}). Therefore \textsc{EvaluateEdges} will eventually evaluate all edges in \graph, in which case \textsc{ComputePath} will have access to the true costs of all edges and will add a feasible path to \Pvec if such a path exists. Lemma~\ref{lem:monitor_paths} guarantees that the path will then be returned as the solution by \textsc{MonitorPaths} in finite time.
\end{proof}

\begin{lemma}
\label{lem:cost_increasing}
The cost $\cost(\plan_i)$ of a path $\plan_i$ computed by \textsc{ComputePath} in any iteration $i$ of MPLP (Line~\ref{alg:search/compute_path}, Alg.~\ref{alg:search}) satisfies the relationship $ \cost(\plan_i) \leq \wh\cdot\costopt$, where \costopt is the cost of the optimal path in \graph.
\end{lemma}
\begin{proof}
At any iteration $i$ of MPLP, the search runs on an implicit snapshot of \graph in which the true costs of some of the edges  are known and the remaining edges have an estimated cost which is an underestimate of the true cost. Let this intermediate graph be $\graph_i$. Let the cost of an optimal path in \graph be \costopt, and the cost of the same path in $\graph_i$ be $\cost_i$. Since the cost of the unevaluated edges in $\graph_i$ are an underestimate of the corresponding edges in \graph, this implies that $\cost_i \leq \costopt$. Let the cost of an optimal path in $\graph_i$ be $\cost^*_i$. Therefore,  $\cost^*_i\leq\cost_i\leq\costopt$.
Since \textsc{ComputePath} runs weighted A* on $\graph_i$, the cost  of any path $\plan_i$ computed by it in any iteration $i$ satisfies $\cost(\plan_i) \leq \wh\cdot\cost^*_i$. Therefore, 

\begin{align*}
    &\implies\cost(\plan_i) \leq \wh\cdot\cost^*_i \leq \wh\cdot\cost_i \leq \wh\cdot\costopt \\   
\end{align*}

\end{proof}

\begin{theorem}[\textbf{Soundness}]
\label{th:soundness}
The path returned by MPLP is fully evaluated and feasible.
\end{theorem}
\begin{proof}
\textsc{MonitorPaths} only returns a fully evaluated path that has true cost no greater than \goalg (Line~\ref{alg:monitor_paths/feasibility_check} , Alg.~\ref{alg:monitor_paths}). For it to return an infeasible path, \goalg has to be $\infty$. However since \goalg is initialized to $-\infty$ and gets updated when \textsc{ComputePath} finds a path (Line~\ref{alg:search/goalg_update} , Alg.~\ref{alg:search}), for it to have value of $\infty$, \textsc{ComputePath} has to find a path with $\infty$ cost. This is not possible because Lemma~\ref{lem:cost_increasing} states that any path returned by \textsc{ComputePath} has a finite upper bound, if the graph has a feasible solution.
\end{proof}

\begin{theorem}[\textbf{Bounded suboptimality}]
\label{th:optimality}
The path \plan returned as the solution by MPLP satisfies $\costreal(\plan) \leq \wh\cdot\costopt$ where \costopt is the cost of the optimal path in \graph.
\end{theorem}
\begin{proof}
As per Lemma~\ref{lem:cost_increasing}, a path $\plan_i$ returned by \textsc{ComputePath} in any iteration $i$ of \textsc{MPLP} will never have a cost greater that $\wh\cdot\costopt$. Therefore, 
\begin{align*}
    \implies&\max_i \cost(\plan_i)=\goalg\leq~\wh\cdot\costopt
\end{align*}
If a fully evaluated path \plan is found in \textsc{MonitorPaths} such that $\costreal(\plan)~\leq~\goalg$ then $\costreal(\plan)~\leq~\wh\cdot\costopt$.
\end{proof}

%% file: 06evaluation.tex
We evaluate MPLP in two planning domains where edge evaluation is expensive. To emphasize the computational budget, we append the number of threads (\numthreads) being used by MPLP as a suffix i.e. MPLP-\numthreads. \sm{MPLP and the baselines were implemented in C++.}

\subsection{3D navigation}


The first domain is motion planning for 3D ($x,y,\theta$) navigation of a PR2 robot in an indoor environment similar to the one used in \cite{narayanan2017heuristic} and shown in Fig.~\ref{fig:nav3d_problem}. The robot can move along 18 simple motion primitives that independently change the three state coordinates by incremental amounts. Evaluating each primitive involves collision checking of the robot model (approximated as spheres) against the world model (represented as a 3D voxel grid) at interpolated states on the primitive. Though approximating the robot with spheres instead of meshes dramatically speeds up collision checking, it is still the most expensive component of the search. The computational cost of edge evaluation increases with an increasing granularity of interpolated states at which collision checking is carried out. \sm{We use two types of primitives: 1) 16 primitives that change the $(x, y)$ coordinates, but do not change $\theta$ and 2) 2 primitives that only change $\theta$. We vary the computational cost of edge evaluations by varying the Euclidean distance (\cci) between two consecutive states along a primitive at which collision checking is carried out (i.e. the discretization of the primitives for collision checking), for the first type of primitives. In general, a smaller \cci produces a better approximation, while larger \cci can cause the robot to tunnel through obstacles. For the second type of primitives, collision checking is always carried out at $\pm1^{\circ}$ increments in $\theta$, and this parameter is not varied for the sake of simplicity.} In the optimistic approximation of the primitives, we collision check just the final state along the primitive. \sm{In this domain, this approximation is incorrect about $24\%$ of the time.} The search uses Euclidean distance as the admissible heuristic. \sm{The experiments were run on an AWS c5a.24xlarge instance with 96 vCPUs, running Ubuntu 18.04. We evaluate on 50 trials in each of which the start configuration of the robot and goal region are sampled randomly.}

\begin{figure}[ht]
    \centering
    \includegraphics[width=\columnwidth]{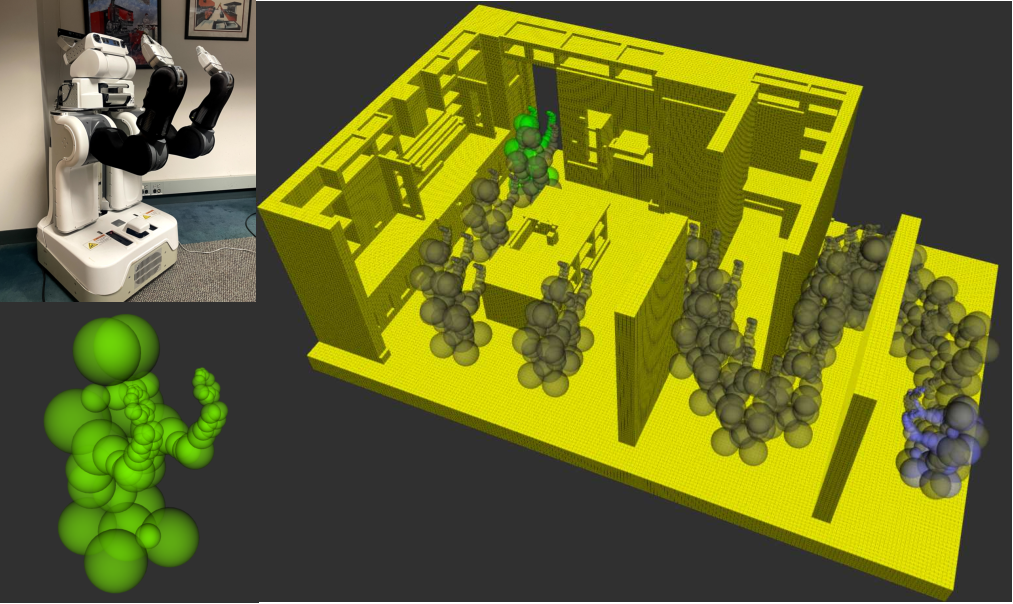}
    \caption{(\textbf{Navigation}) Left: The PR2's collision model is approximated with spheres. 
    Right: The task is to navigate in an indoor map from a given start (purple) and goal (green) states using a set of motion primitives. States at the end of every primitive in the generated plan are shown in black.}
    \label{fig:nav3d_problem}
\end{figure}

\subsubsection{Comparison to lazy search baselines}

We compare MPLP-90 ($\numthreads=90$) with weighted A* and lazy search baselines LwA*~\cite{cohen2014planning} and LSP~\cite{dellin2016unifying} (which are instantiations of GLS~\cite{mandalika2019generalized}). Lazy search algorithms are designed to increase efficiency in domains where edge evaluations are expensive. Therefore we analyze the performance gain achieved by MPLP with increasing computational cost of edge evaluations by reducing \cci. Fig.~\ref{fig:lazy_plots} shows the average speedup achieved by  MPLP  and the baselines over wA* for varying \cci on a set of start and goal pairs with uninflated and inflated heuristics. Speedup over wA* is defined as the ratio of the average runtime of wA* over the average runtime of a specific algorithm. The corresponding raw data is shown in Table~\ref{tab:nav3d_lazy_times}. With increasing granularity of collision checking (decreasing \cci), the speedup achieved by MPLP rapidly outpaces that of the baselines.

\begin{figure}[ht]
    \centering
    \includegraphics[width=\columnwidth]{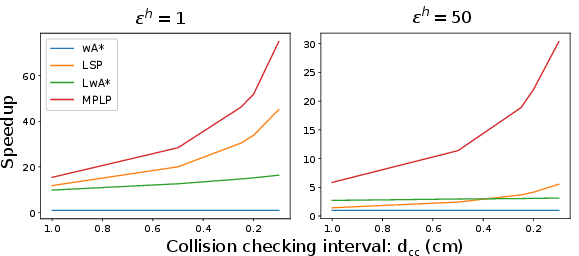}
    \caption{(\textbf{Navigation}) Average speedup achieved by MPLP,  LSP and LwA* over wA*  with uninflated heuristic (left) and with a heurisitic inflation of 50 (right). \cci decreases along the x-axis which increases edge evaluation time.}
    \label{fig:lazy_plots}
\end{figure}

\begin{table}[]
\begin{subtable}{\columnwidth}
\centering
\resizebox{\columnwidth}{!}{%
\begin{tabular}{cccccc}
\toprule
& \multicolumn{5}{c}{Collision checking interval: \cci (\si{\cm})} \\ \midrule
& 1      & 0.5     & 0.25    & 0.2    & 0.1    \\\midrule
&\multicolumn{5}{c}{$\wh=1$} \\ \midrule
wA*                  & 3.55     & 6.82      & 13.40     & 16.59    & 32.96    \\
LwA*                 & 0.36     & 0.54      & 0.91      & 1.09     & 2.01     \\
LSP                  & 0.30     & 0.34      & 0.44      & 0.49     & 0.73     \\
MPLP-90              & \textbf{0.23}     & \textbf{0.24}      & \textbf{0.29}      & \textbf{0.32}     & \textbf{0.44}     \\
\midrule
\end{tabular}
}
\end{subtable}
\begin{subtable}{\columnwidth}
\centering
\resizebox{\columnwidth}{!}{%
\begin{tabular}{cccccc}
& \multicolumn{5}{c}{$\wh=50$} \\ \midrule
wA*                  & 0.76     & 1.48      & 2.83      & 3.53     & 6.98    \\
LwA*                 & 0.28     & 0.50      & 0.93      & 1.15     & 2.24     \\
LSP                  & 0.53     & 0.61      & 0.77      & 0.85     & 1.26     \\
MPLP-90              & \textbf{0.13}     & \textbf{0.13}      & \textbf{0.15}      & \textbf{0.16}     & \textbf{0.23}     \\
\bottomrule
\end{tabular}
}
\end{subtable}
\caption{(\textbf{Navigation}) Average planning times (\si{\second}) for MPLP, wA* and lazy search baselines for varying \cci, with and without heuristic inflation.}
\label{tab:nav3d_lazy_times}
\end{table}

\subsubsection{Comparison to parallel search baselines}
We also compare MPLP with parallel search baselines. The first baseline is a variant of weighted A* in which during a state expansion, the successors of the state are generated and the corresponding edges are evaluated in parallel. For lack of a better term, we call this baseline Parallel Weighted A* (PwA*). Note that this is very different from the Parallel A* (PA*) algorithm~\cite{irani1986parallel}. The second baseline is wPA*SE~\cite{phillips2014pa}. These two baselines leverage parallelization differently. In PwA*, parallelization is at the level of generation of successors, whereas in wPA*SE parallelization is at the level of state expansions. Fig.~\ref{fig:parallel_plots} shows the average speedup achieved by MPLP and the baselines over wA* for varying number of threads, with uninflated and inflated heuristics and with two different values of \cci. The corresponding raw data is shown in Table~\ref{tab:nav3d_parallel_times}. For a single thread, PwA* and wPA*SE have the same runtime as that of wA*. As described in Section~\ref{sec:methods}, MPLP needs a minimum of 4 threads. Since PwA* parallelizes successor generation during an expansion, increasing the number of threads beyond a certain point does not lead to any performance improvement. The maximum speedup achieved by wPA*SE is dependent on the number of states that can be safely expanded in parallel. The performance degrades with a higher number of threads which is consistent with what was observed in ~\cite{phillips2014pa}. Consistent with what was observed in the comparison with lazy search baselines, the speedup obtained by MPLP is greater for a smaller \cci (larger edge evaluation computational cost). In addition, MPLP's performance gain saturates at a higher \numthreads for a smaller \cci. This shows that MPLP leverages multithreading more effectively with increasing computational cost of edge evaluations. \sm{MPLP is effective in domains where the computational cost of evaluating edges relatively outweighs that of exploring the graph optimistically. This implies smaller graphs with expensive to evaluate edges. As is the case with all lazy search algorithms, in domains with larger graphs and inexpensive edges, MPLP is not effective. This can be observed in the top two plots in Fig.~\ref{fig:parallel_plots}. With $\cci = 1\si{\cm}$, wPA*SE outperforms MPLP. However with more expensive edges, as is the case in the bottom two plots, MPLP comprehensively outperforms wPA*SE.}

\begin{figure}[ht]
    \centering
    \includegraphics[width=\columnwidth]{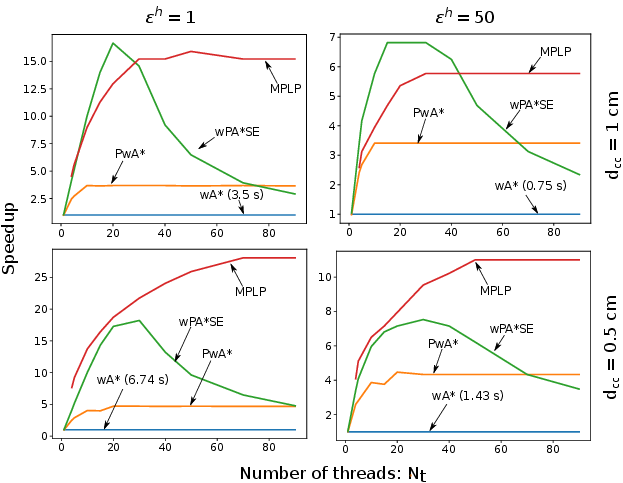}
    \caption{(\textbf{Navigation}) Average speedup achieved by MPLP, PwA* and wPA*SE over wA*  with uninflated heuristic (left) and with a heuristic inflation factor of 50 (right). MPLP achieves a significantly higher speedup as compared to the baselines when $\cci = 0.5\si{\cm}$.}
    \label{fig:parallel_plots}
\end{figure}

\begin{table*}[]
\begin{subtable}{\textwidth}
\centering
\resizebox{\columnwidth}{!}{%
\begin{tabular}{cccccccccccc}
\toprule
           \multicolumn{12}{c}{Number of threads (\numthreads)}       \\\midrule
          & 1    & 4    & 5    & 10   & 15   & 20   & 30   & 40   & 50   & 70   & 90   \\\midrule
          & \multicolumn{10}{c}{$\cci = 1\si{\cm}~|~\wh = 1$}    \\\midrule
wA*       & 3.5  & -    & -    & -    & -    & -    & -    & -    & -    & -    & -    \\
PwA*      & 3.54 & 1.43 & 1.29 & 0.95 & 0.96 & 0.95 & 0.95 & 0.95 & 0.96 & 0.95 & 0.96 \\
wPA*SE     & 3.40 & 0.85 & 0.69 & 0.35 & 0.25 & 0.21 & 0.24 & 0.38 & 0.54 & 0.89 & 1.20 \\
MPLP      & -    & \textbf{0.77} & \textbf{0.63} & \textbf{0.39} & \textbf{0.31} & \textbf{0.27} & \textbf{0.23} & \textbf{0.23} & \textbf{0.22} & \textbf{0.23} & \textbf{0.23} \\
\midrule
\end{tabular}
}
\end{subtable}
\begin{subtable}{\textwidth}
\centering
\resizebox{\columnwidth}{!}{%
\begin{tabular}{cccccccccccc}
          & \multicolumn{10}{c}{$\cci = 1\si{\cm}~|~\wh = 50$}    \\\midrule
wA*       & 0.75  & -    & -    & -    & -    & -    & -    & -    & -    & -    & -    \\
PwA*      & 0.76 & 0.31 & 0.28 & 0.22 & 0.22 & 0.22 & 0.22 & 0.22 & 0.22 & 0.22 & 0.22 \\
wPA*SE     & 0.72 & 0.22 & 0.18 & 0.13 & 0.11 & 0.11 & 0.11 & 0.12 & 0.16 & 0.24 & 0.32 \\
MPLP   & -    & \textbf{0.29} & \textbf{0.24} & \textbf{0.19} & \textbf{0.16} & \textbf{0.14} & \textbf{0.13} & \textbf{0.13} & \textbf{0.13} & \textbf{0.13} & \textbf{0.13} \\
\midrule
\end{tabular}
}
\end{subtable}
\begin{subtable}{\textwidth}
\centering
\resizebox{\columnwidth}{!}{%
\begin{tabular}{cccccccccccc}
          & \multicolumn{10}{c}{$\cci = 0.5\si{\cm}~|~\wh = 1$}    \\\midrule            
WA*       & 6.74  & -    & -    & -    & -    & -    & -    & -    & -    & -    & -    \\
PwA*      & 6.76 & 2.68 & 2.33 & 1.67 & 1.69 & 1.43 & 1.43 & 1.44 & 1.43 & 1.44 & 1.44  \\
PA*SE     & 6.44 & 1.64 & 1.30 & 0.67 & 0.47 & 0.39 & 0.37 & 0.51 & 0.70 & 1.04 & 1.41 \\
MPLP   & -     & \textbf{0.88} & \textbf{0.73} & \textbf{0.49} & \textbf{0.41} & \textbf{0.36} & \textbf{0.31}  & \textbf{0.28}  & \textbf{0.26}  & \textbf{0.24}  & \textbf{0.24} \\
\midrule
\end{tabular}
}
\end{subtable}
\begin{subtable}{\textwidth}
\centering
\resizebox{\columnwidth}{!}{%
\begin{tabular}{cccccccccccc}
          & \multicolumn{10}{c}{$\cci = 0.5\si{\cm}~|~\wh = 50$}    \\\midrule
WA*       & 1.43  & -    & -    & -    & -    & -    & -    & -    & -    & -    & -    \\
PwA*      & 1.42 & 0.55 & 0.51 & 0.37 & 0.38 & 0.32 & 0.33 & 0.33 & 0.33 & 0.33 & 0.33 \\
PA*SE     & 1.37 & 0.42 & 0.35 & 0.24 & 0.21 & 0.20 & 0.19 & 0.20 & 0.23 & 0.33 & 0.41 \\
MPLP   & -    & \textbf{0.35} & \textbf{0.28} & \textbf{0.22} & \textbf{0.20} & \textbf{0.18} & \textbf{0.15} & \textbf{0.14} & \textbf{0.13} & \textbf{0.13} & \textbf{0.13} \\
\bottomrule
\end{tabular}
}
\end{subtable}
\caption{(\textbf{Navigation}) Average planning times (\si{\second}) for MPLP, wA* and parallel search baselines (PwA* and PA*SE) for varying \numthreads, with different values of \cci and \wh.}
\label{tab:nav3d_parallel_times}
\end{table*}

\subsubsection{Ablation of priority inflation of edges in optimistic paths}
\label{sec:priority_ablation}
As discussed earlier, to prioritize evaluation of edges that are in the paths computed by the optimistic search over the other discovered edges, their priorities in \eopen are dynamically increased (Line~\ref{alg:search/inflate_priority}, Alg.~\ref{alg:search}). We ablate this to highlight the benefit of doing so and paying the cost of rebalancing \eopen. Fig.~\ref{fig:priority_ablation_plots} shows the average planning times of MPLP with and without the priority inflation, for varying number of threads, with and without heuristic inflation. For smaller \numthreads, priority inflation significantly reduces planning time. With increasing \numthreads, the performance gain diminishes since there is enough computational resource available to evaluate all edges in \eopen parallelly without having to preferentially evaluate edges in the optimistic paths. At the same time, with heuristic inflation, the performance gain of priority inflation is also lower. This is because the search is greedy and discovers fewer edges that need to be evaluated.

\begin{figure}[ht]
    \centering
    \includegraphics[width=\columnwidth]{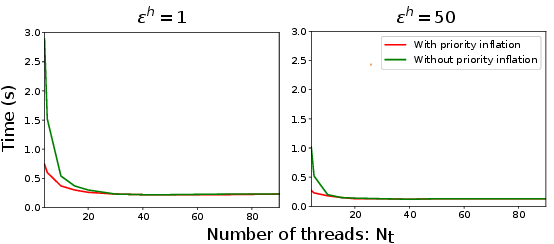}
    \caption{(\textbf{Navigation}) Average planning time (\si{\second}) of MPLP with and without edge priority inflation for varying number of threads.}
    \label{fig:priority_ablation_plots}
\end{figure}

\subsection{Assembly task}

\begin{figure}[ht]
    \centering
    \includegraphics[width=\columnwidth]{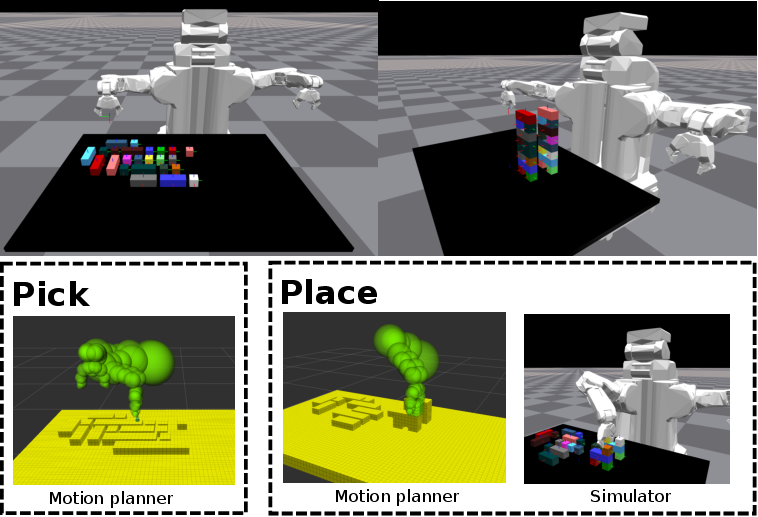}
    \caption{(\textbf{Assembly}) Top: The PR2 has to arrange a set of blocks on the table (left) into a given configuration (right). Bottom: It is equipped with \pick and \place controllers. The \pick controller uses the motion planner to reach a block. The \place controller uses the motion planner to place a block and simulates the outcome of releasing the block.}
    \label{fig:assembly_problem_controllers}
\end{figure}

\begin{table}[]
\centering
\resizebox{\columnwidth}{!}{%
\begin{tabular}{c|c|c|c|c}
\toprule
                           & wA*  & LwA* & LSP & MPLP-40 \\\midrule
Time (s)                   & 496  & 259  & 133 &  \textbf{84}  \\
Speedup                    & 1    & 1.9  & 3.7 &  \textbf{5.9} \\\bottomrule
\end{tabular}
}
\caption{(\textbf{Assembly}) Average planning time, speedup over wA*  of MPLP compared to those of the lazy search baselines.}
\label{tab:assembly_times}
\end{table}

The second domain is a task and motion planning problem of assembling a set of blocks on a table into a given structure by a PR2, as shown in Fig.~\ref{fig:assembly_problem_controllers}. We assume full state observability of the 6D poses of the blocks and the robot's joint configuration. The goal is defined by the 6D poses of each block in the desired structure. The PR2 is equipped with \pick and \place controllers which are used as macro-actions in the high-level planner. Both of these actions use a motion planner internally to compute collision-free trajectories in the workspace. Additionally, \place  has access to a simulator (NVIDIA Isaac Gym~\cite{makoviychuk2021isaac}) to simulate the outcome of placing a block in its desired pose. For example, if the planner tries to place a block in its final pose but has not placed the block underneath yet, the placed block will not be supported and the structure will not be stable. This would lead to an invalid successor during planning. \sm{We set a simulation timeout of  $\tsim=0.2$~\si{\second} to evaluate the outcome of placing a block. Considering the variability in the simulation speed and the overhead of communicating with the simulator, this results in a total wall time of less than $1$~\si{\second} for the simulation.  The motion planner has a timeout of $\tplan=60$~\si{\second} based on the wall time, and therefore that is the maximum time the motion planning can take.} Since the workspace is cluttered, the bottleneck in this domain is the motion planning component of these actions. In the optimistic approximation, these macro-actions are replaced by their optimistic versions which substitute the motion planner with an IK solver, while the motion planner is used when the corresponding edges are evaluated. Successful \pick and \place actions have unit real and optimistic costs, and infinite otherwise. A \pick action on a block is successful if the motion planner finds a feasible trajectory to reach the block within \tplan. A \place action on a block is successful if the motion planner finds a feasible trajectory to place the block within \tplan and simulating the block placement results in the block coming to rest at the desired pose within \tsim. The cost of every feasible plan is twice the number of blocks since placing a block in its desired pose involves a single \pick and \place controller pair each of which has a unit cost. \sm{The experiments were run on an AWS g4dn.16xlarge instance with 64 vCPUs, running Ubuntu 18.04. The instance also has an NVIDIA T4 Tensor Core GPU to run the simulator.} MPLP is run with $\numthreads = 40$ and the addition of more threads did not improve performance in this domain. The number of blocks that are not in their final desired pose is used as the admissible heuristic, with an inflation factor of 5. Table~\ref{tab:assembly_times} shows planning times, speedup over wA*  of MPLP-40 as compared to those of the lazy search baselines.  The numbers are averaged across 20 trials in each of which the blocks are arranged in random order on the table. MPLP-40 achieves a 5.9x speedup over wA* and a 1.6x speedup over LSP.

%% file: 07conclusions.tex
In this work, we presented MPLP, a massively parallelized lazy search algorithm that integrates ideas from lazy search and parallel search. We proved that MPLP is sound, complete and bounded suboptimal. Our experiments showed that MPLP achieves higher efficiency than both lazy search and parallel search algorithms on two very different planning domains. \sm{Therefore in practice, we recommend using MPLP in any planning domain where the computational bottleneck is edge evaluations and where the successors of a state can be generated without evaluating the connecting edges.}

MPLP assigns a uniform evaluation priority to edges when they are first discovered and only increases the priorities of edges that belong to a path. However, if there is a domain-dependent edge-existence prior available, it can be seamlessly integrated into the algorithm by assigning the evaluation priority derived from it. 

Instead of a naive implementation of LSP, where each shortest path search is run from scratch, an incremental approach of updating the graph and re-using the previous search tree using LPA* mechanics can be more efficient \cite{mandalika2019generalized, lim2021lazy}. However, because of the massive parallelization of edge evaluations in MPLP, a potentially large number of edges get updated in-between searches. Therefore running the search in each iteration from scratch is more efficient than incremental methods. However, based on the number of updated edges, an adaptive strategy can be employed to either re-use the previous search tree or plan from scratch.